\documentclass[10pt,twocolumn,letterpaper]{article}

\usepackage[pagenumbers]{cvpr}              

\usepackage{multirow}
\usepackage{siunitx}
\usepackage{booktabs}
\usepackage{wrapfig}
\usepackage{sidecap}
\sidecaptionvpos{figure}{c}
\usepackage{amsthm}          

\newtheorem{lemma}{Lemma}    

\newtheorem*{lemma*}{Lemma}  
\newtheorem*{proof*}{Proof}  
\usepackage[utf8]{inputenc} 
\usepackage[T1]{fontenc}    
\usepackage{url}            
\usepackage{booktabs}       
\usepackage{amsfonts}       
\usepackage{nicefrac}       
\usepackage{microtype}      
\usepackage{xcolor}         
\usepackage{subcaption}
\usepackage{amsmath}
\usepackage{graphicx}

\definecolor{cvprblue}{rgb}{0.21,0.49,0.74}
\usepackage[pagebackref,breaklinks,colorlinks,allcolors=cvprblue]{hyperref}

\def\paperID{8825} 
\def\confName{CVPR}
\def\confYear{2026}

\title{DiffGradCAM: A Class Activation Map Using the Full Model Decision to Solve Unaddressed Adversarial Attacks}

\author{%
  Jacob Piland$^\dag$, Christopher Sweet$^\ddag$, and Adam Czajka$^\dag$ 
  \\~$^\dag$Department of Computer Science and Engineering,~$^\ddag$Center for Research Computing\\
  University of Notre Dame du Lac, Notre Dame, IN 46556, USA\\
  {\tt\small \{jpiland,csweet1,aczajka\}@nd.edu}
}

\begin{document}
\maketitle

\begin{abstract}

Class Activation Mapping (CAM) and its gradient-based variants (\eg, GradCAM) have become standard tools for explaining Convolutional Neural Network (CNN) predictions. 
However, these approaches typically focus on individual logits, while for neural networks using softmax, the class membership probability estimates depend only on the differences between logits, not on their absolute values.
This disconnect leaves standard CAMs vulnerable to adversarial manipulation, such as passive fooling, where a model is trained to produce misleading CAMs without affecting decision performance.

To address this vulnerability, we propose DiffGradCAM and its higher-order derivative version DiffGradCAM++, as novel, lightweight, contrastive approaches to class activation mapping that are not susceptible to passive fooling and match the output of standard methods such as GradCAM and GradCAM++ in the non-adversarial case.
To test our claims, we introduce Salience-Hoax Activation Maps (SHAMs), a more advanced, entropy-aware form of passive fooling that serves as a benchmark for CAM robustness under adversarial conditions. Together, SHAM and DiffGradCAM establish a new framework for probing and improving the robustness of saliency-based explanations. We validate both contributions across multi-class tasks with few and many classes.

\end{abstract}

\section{Introduction}

\subsection{Background and Motivation}

Interpretability methods for deep neural networks are critical for ensuring trust, transparency, and accountability in machine learning systems. Among them, Class Activation Mapping (CAM) \cite{Zhou_CVPR_2016} and its gradient-based extensions such as Gradient-weighted Class Activation Mapping (GradCAM) \cite{gradcam} have become standard techniques for visualizing which regions of an input most influence a CNN's prediction. 
However, standard CAMs rely on a simplifying assumption that the importance of a class can be understood by inspecting the gradient of its individual logit. In contrast, softmax decisions depend on \emph{logit differences} and not their absolute values. For example, in binary classification, the model's output is governed by the difference $y_1 - y_2$ (where $y_1$ and $y_2$ are the logits of two output neurons), not the magnitude of $y_1$ alone. Focusing on a single logit therefore observes only part of the model’s reasoning; aggregating over the \emph{entire competitor set} would integrate both supporting and opposing evidence, aligning explanations with the actual decision the network makes.

This fundamental disconnect between what CAM methods visualize and how decisions are actually made introduces a critical adversarial vulnerability: 
CAMs can be passively fooled, that is, a model can be adversarially trained or fine-tuned to produce misleading CAMs while preserving predictive accuracy \cite{heo2019fooling}. Yet these prior manipulations are self-described as arbitrary, and do not necessarily take into account behavioral details that are known about models trained with salience, making them less representative of real-world model behavior.

\subsection{Proposed Approach and Research Questions}
The above vulnerability is addressed in this paper in a twofold way. First, \textbf{we introduce Salience-Hoax Activation Maps (SHAMs)} as a benchmark. SHAMs are a form of adversarial salience that, when used in training or fine-tuning, produce an entropy-aware form of passive fooling. It has been shown that different models have different average CAM entropies \cite{DROID}, and thus SHAMs improve on previous adversarial techniques by taking into account CAM entropy of models trained with salience-based training in their design. Models trained or fine-tuned with adversarial SHAM saliency maps maintain performance and also the expected model CAM entropy while redirecting activations to meaningless regions (\eg, image borders instead of salient features). Because SHAMs preserve model accuracy while generating realistic yet misleading explanations, they provide both a generalizable threat model and a broadly applicable benchmarking tool for evaluating the robustness of interpretation methods.

Second, to address the vulnerability exposed by benchmark SHAMs, which expand on previous adversarial techniques, \textbf{we propose DiffGradCAM} and its higher-order derivative counterpart DiffGradCAM++. Both are lightweight, contrastive variants that align saliency with the model’s decision. Rather than targeting a single logit, DiffGradCAM computes gradients with respect to the difference between the true class logit and an aggregate over the competing logits. This contrastive formulation directly reflects the softmax decision boundary and enables more faithful saliency.

We examine several candidate aggregation functions for DiffGradCAM and DiffGradCAM++ and find that MeanDiffGradCAM and MeanDiffGradCAM++, variants using the mean of false-class logits as the contrast baseline, are capable of producing the same mappings as GradCAM and GradCAM++ in the non-adversarial setting, and they exhibit significantly improved resistance to SHAM-based manipulation. We evaluate the effects of SHAM and the robustness of DiffGradCAM across multi-class tasks with few-class and many-class settings, demonstrating that SHAM alters GradCAM and GradCAM++, and that DiffGradCAM and DiffGradCAM++ provide a practical, CNN-architecture-agnostic path to adversarially robust explanation. We define the following {\bf research questions (RQs)} to systematically evaluate the effectiveness and robustness of SHAM and DiffGradCAM:

\begin{itemize}[leftmargin=9.5mm]
    \itemsep0em
    \item[{\bf RQ1:}] In the novel, few-class iris presentation attack detection (PAD) domain, can SHAM-based passive fooling produce misleading CAMs without negatively impacting classification accuracy?

    \item[{\bf RQ2:}] In the standard, large-scale ImageNet setting, do any of the proposed DiffGradCAM variants match GradCAM in a non-adversarial context?

    \item[{\bf RQ3:}] In the large-scale ImageNet scenario, are DiffGradCAM variants resistant to SHAM-based adversarial manipulation, and if so, how does their resistance compare to class-independent CAM methods?
\end{itemize}

\subsection{Summary of Contributions}

We introduce SHAMs as an entropy-aware generalization of passive fooling that leverages CAM entropy from salience-based training to induce misleading CAMs without degrading predictive performance.
We propose DiffGradCAM and DiffGradCAM++, which replace single-logit targets with a decision-aligned logit difference $\Delta^c$ (Eqn.~\ref{eqn:delta})
, thereby aggregating over all competitors and capturing more of the model’s reasoning. We qualitatively evaluate an existing class-agnostic CAM variant, against our DiffGradCAM, on a few-class Iris PAD task.
Finally, we quantitatively assess DiffGradCAM and DiffGradCAM++’s fidelity to GradCAM and GradCAM++ in non-adversarial settings and their robustness to SHAM-based passive fooling on ImageNet, on a variety of CNN backbones showing that our novel methods can replace GradCAM and GradCAM++ while offering stronger resistance to adversarial CAM manipulation.

\section{Related Works}

Since CAMs were first introduced in 2016 \cite{Zhou_CVPR_2016}, there has been an explosion of CAM alternatives \cite{hirescam,scorecam,gradcam++,ablationcam,xgradcam,eigencam} with one of the most popular being GradCAM \cite{gradcam}. However, it has been shown that models can be deceived into producing arbitrary CAMs, either through adversarial input \cite{input_manip} or training manipulation \cite{heo2019fooling}. The latter category is further divided into active fooling, where the model is trained so that CAMs produced on samples from one class resemble those of another class, and passive fooling, where the model is trained to produce a nonsense CAM regardless of input. {\bf This work differs from prior approaches} by providing an updated passive-fooling method that uses published CAM entropy \cite{cam_entropy2022} to construct an adversarial salience that misleads CAMs without degrading performance, and by introducing DiffGradCAM, a post-hoc, lightweight, general variant of GradCAM that is robust to adversarial CAM manipulation while matching GradCAM in non-adversarial settings.

\section{Limitations of Standard CAM/GradCAM}
\label{sec:limitaion_of_standard_CAM}
\subsection{The Dominant Logit Assumption}
\label{sec:the_dominant_logit_assumption}
Standard CAM/GradCAM weight feature maps by the gradient of the true-class logit with respect to activations, assuming this gradient isolates regions important for that class. The mathematical formulation relies on an implicit assumption: when the true logit greatly exceeds other logits, the gradient of the true logit dominates the explanation.

Because softmax outputs sum to one, the gradient of the true logit equals the negative sum of gradients of all false-class logits. When the true logit substantially exceeds others, this yields a small combined contribution from non-target classes. Consequently, false-class activations minimally influence the resulting heatmap.

A critical insight motivating our approach is that for neural networks using softmax, the probability output depends \textit{only} on the \textit{differences} between logits, not on their absolute values. As an example, we make the following observation for the binary class case.

\begin{lemma}
    In the binary classification setting with logits $y_1$ and $y_2$, softmax reduces to the sigmoid function applied to their difference.
\end{lemma}

\begin{proof}
    The definition of softmax for the two-class case and factor \( e^{y_2} \) out of both numerator and denominator is:
    \[
    p_1 = \frac{e^{y_1}}{e^{y_1} + e^{y_2}} = \frac{e^{y_1} / e^{y_2}}{e^{y_1}/e^{y_2} + 1} = \frac{e^{y_1 - y_2}}{e^{y_1 - y_2} + 1}
    \]
    which can be expressed as the sigmoid function:
    \[
    p_1 = \frac{1}{1 + e^{-(y_1 - y_2)}} = \sigma(y_1 - y_2)
    \]
\end{proof}

This demonstrates that the model's decision fundamentally depends on the logit difference $(y_1 - y_2)$, not on individual logit values. Standard CAM/GradCAM approaches, by focusing only on the gradient of the true-class logit, fail to explicitly capture this differential relationship that actually drives model predictions.

\subsection{Salience-Hoax Activation Maps}
\label{ssec:sham}

This reliance on single-logit targeting can fail under adversarial or noisy conditions. Small perturbations that inflate a false logit may not flip the classification outcome but can dramatically alter model behavior. Standard CAM methods may miss these manipulated regions, providing misleading explanations of the model's decision process.

These vulnerabilities have been leveraged to produce arbitrary CAMs \cite{heo2019fooling}. We extend this idea one step further by using known properties of models trained with salience to design an adversarial salience for training models to produce misleading CAMs. We call this adversarial salience a Salience-Hoax Activation Map (SHAM) and models trained or fine-tuned with SHAM are passively-fooled models.

\begin{SCfigure}[2.9][!h]
    \includegraphics[width=0.11\textwidth]{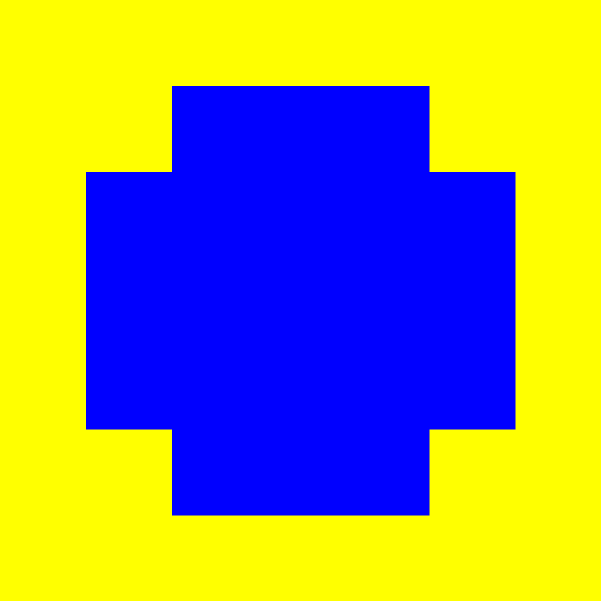}
    \caption{Adversarial SHAM saliency used as saliency in training and fine-tuning for producing misleading CAMs. Yellow and blue sections denote maximum and minimum values, respectively.}
    \label{fig:sham_annotation}
\end{SCfigure}

We choose to direct the model to the edges of the image as in \cite{heo2019fooling} (the assumed worst-case stress-test) and we use CAM entropy to determine how much of the image the adversarial salience annotates. It has previously been shown that some salience-based models produce CAMs with an average CAM entropy of 3.35 Shannons \cite{DROID} and we set our SHAM to match this value.
For DenseNet, which has a $7\times7=49$ pixel CAM, the SHAM has 28 of 49 pixels set to a value of 1 around the edges with the center pixels being set to 0 for a SHAM entropy of 3.33, as illustrated in Fig. \ref{fig:sham_annotation}.

\section{The DiffGradCAM Approach}
\subsection{From GradCAM to DiffGradCAM}
DiffGradCAM is an improvement building on the foundation of GradCAM \cite{gradcam} which is defined as:
\begin{equation}
\mathrm{GradCAM}^c = \mathrm{ReLU}\left(\sum^K_{k=1} \alpha^{c}_k A^k\right),
\end{equation}
where $c$ is the class of interest, $A^k$ is the $k^{th}$ feature map of a total of $K$ feature maps. The coefficients are defined as:
\begin{equation}
\alpha^{c}_k = \frac{1}{Z}\sum^u_{i=1}\sum^v_{j=1}\frac{\partial y^c}{\partial A^k_{i,j}},
\end{equation}

\noindent where $C$ is the total number of logits, $y^c$ the logit considered, $u,v$ are the dimensions of $A$, and $Z$ is the number of feature map elements. 

The idea of DiffGradCAM is to leverage information from the logits associated with the other classes $c' \ne c$ by replacing the logit $y^c$ associated with class $c$ with a function that aggregates the logits based on the class choice $c$. Inspired by the observations in Sec. \ref{sec:the_dominant_logit_assumption}, we denote this function as $\Delta^c$ in the following discussion and the equation for the DiffGradCAM coefficients becomes:
\begin{equation}
\alpha^{c}_k = \frac{1}{Z}\sum_i\sum_j\frac{\partial \Delta^c}{\partial A^k_{i,j}}.
\end{equation}

\label{sec:dgc}
\subsection{The binary classification case}
In the binary case $\Delta^c$ simplifies to: 

\begin{equation}
\Delta^c = y^c - y^{c'},
\end{equation}
where $c' \ne c$.

Taking gradients $\partial \Delta / \partial A$ directly measures how each feature-map activation $A$ shifts the model's preference for the true class over its competitors.

By formulating the explanation target as a logit difference, DiffGradCAM aligns the visualization with the actual decision mechanism of the model. This represents a fundamental shift from highlighting regions that merely increase a single logit to highlighting regions that contribute to the model's class discrimination.

\subsection{Extension to Multi-Class Settings}
For multi-class models with more than two classes, we define a contrastive logit for each class $c$:

\begin{equation}
\Delta^c = y^c - \beta(Y^c),
\label{eqn:delta}
\end{equation}
where $Y^c=\{y^{c'} : c' \neq c\}$ and $\beta$ is a function over the non-target logits $Y^c$. We consider three $\beta$ functions: 

\begin{equation}
\beta_{\mathrm{mean}}^c = \frac{1}{C-1}\sum_{j\neq c} y^j,
\end{equation}
\noindent called ``Mean baseline,'' where $C$ is the number of logits,

\begin{equation}
\beta_{\mathrm{max}}^c = \max_{j\neq c} y^j,
\end{equation}
\noindent called ``Max baseline,'' and

\begin{equation}
\beta_\mathrm{LSE}^c = \log \left(\frac1{C-1}\sum_{j\neq c} e^{y^j}\right),
\end{equation}
\noindent called ``log-sum-exp (LSE) baseline.'' The shift $\frac1{C-1}$ just matches the scale of the MeanDiffGradCAM baseline, and it does not alter gradients or variance.


$\Delta^c$ isolates evidence that lifts the true class above its rivals.  
Subtracting the mean of the false logits measures the margin over a typical competitor, while the smooth LSE baseline 
weighs the strongest rival most.  
The Mean baseline suits spread-out residuals, while the log-sum-exp baseline suits cases where one non-target class dominates.  
We analyze this trade-off next.

\subsection{Choice of Aggregator}
\label{sec:dist_choice}
The $\Delta^c$ is the differential logit driving DiffGradCAM. We wish to pick the baseline function
$\beta(\cdot)$ so that $\Delta^c$ is \emph{stable} (low variance) yet discriminative (large mean separation) across draws of the false
logits $\{y^{c'} : c' \neq c\}$. We study two canonical choices: $\beta_{\mathrm{mean}}$ and $\beta_{\mathrm{LSE}}$.

The $\beta_{\mathrm{LSE}}$ baseline lies between the mean and the max
(according to Jensen \cite{jensen1906fonctions}: $\beta_{\mathrm{mean}}$ $\le$ $\beta_{\mathrm{LSE}}$ $\le$ $\beta_{\mathrm{max}}$). 
We consider several cases before providing a practical guideline, which is supported by our results (Sec.~\ref{sec:results}).

\subsubsection{Residual-logit Model}
Assume the false logits $Y^c$ are i.i.d.\ from a
distribution with mean $\mu$ and variance $\sigma^{2}$.
Write
\(\hat\mu=\frac1{C-1}\sum_{j\neq c}y^j\) and
\(y_{\max}=\max_{j\neq c}y^j\).
Throughout we treat $C\!>\!2$.

\subsubsection{Case: Many Classes $(C\gg1)$ and Broad Tails}
ImageNet-scale models exhibit large $C$ and sizable
$\sigma^{2}$.  Extreme-value theory \cite{leadbetter2012extremes} then gives
\begin{equation}
\label{eq:evl}
\mathbb{E}[y_{\max}]
\;=\;
\mu + \sigma\sqrt{2\log(C-1)} \;+\; o(1),
\end{equation}
valid for any sub-Gaussian residual distribution.
Consequently $\mathrm{MeanDiffGradCAM}\approx y_{\max}-\log(C-1)
$ (dominated by
the largest term in the sum), which inflates
$\operatorname{Var}[\Delta^c]$ and yields noisy saliency maps.
By contrast, $\beta_{\mathrm{mean}}$ satisfies
$\operatorname{Var}[\beta_{\mathrm{mean}}]=\sigma^{2}/(C-1)$,
shrinking to zero as $C$ grows.

\begin{lemma}[Heavy-tail regime]
\label{lem:mean_opt}
If $\sigma\sqrt{\log C}\gg 1$ (broad residual distribution) then
\(
\operatorname{Var}[\Delta^c]\;
\text{is minimized by}\;
\beta_{\mathrm{mean}}.
\)
\end{lemma}

\begin{proof}[Sketch]
Using~\eqref{eq:evl} and the independence of $y^c$ and $Y^c$,
evaluate $\operatorname{Var}[\Delta^c]$ for each baseline and
keep leading terms in $C$.  See supplementary materials.
\end{proof}

\subsubsection{Case: Few Classes or Peaked Residuals}
Datasets such as Iris PAD ($C=7$) produce residual logits that
cluster tightly (assume $\sigma^{2}=O(1/C)$).
A second-order Taylor expansion of $\beta_{\mathrm{soft}}$ around empirical mean
$\hat\mu$ gives
\(
\beta_{\mathrm{soft}}
=\hat\mu + \tfrac{\sigma^{2}}{2} + O(\sigma^{3}),
\)
so $\beta_{\mathrm{soft}}$ and $\beta_{\mathrm{mean}}$ differ by at most
$O(\sigma^{2})$.

\begin{lemma}[Peaked regime]
\label{lem:soft_equal}
If $\sigma^{2}=O(1/C)$ then
$\bigl|\beta_{\mathrm{soft}}-\beta_{\mathrm{mean}}\bigr|
=O(\sigma^{2})$ and the two baselines yield
indistinguishable $\Delta^c$ up to $O(\sigma^{2})$.
\end{lemma}

\begin{proof}
Apply the cumulant-generating expansion
$\log\mathbb{E}[e^{Z}]=\mu+\frac{\sigma^{2}}{2}+O(\sigma^{3})$
with $Z=y^j-\mu$ and substitute $\hat\mu=\mu+O(\sigma)$. \qedhere
\end{proof}

\subsubsection{Practical Guideline}
With few classes or tight residuals the two baselines coincide, so we may use either. However, for many classes with broad residuals (\eg, ImageNet) we should typically use the mean baseline to damp the $\sqrt{\log C}$ boost, as is supported by the above sections.  This is corroborated by our small and large dataset experiments (Sec. \ref{sec:results}).

\subsection{Concrete Example: 3-Class Case}
Consider a model with logits $(y_1,y_2,y_3)$. For class 1, using the mean baseline:

\begin{equation}
\Delta_1 = y_1 - \frac{y_2 + y_3}{2}
\end{equation}

Similarly, $\Delta_2 = y_2 - (y_1+y_3)/2$ and $\Delta_3 = y_3 - (y_1+y_2)/2$. Applying DiffGradCAM yields three contrastive heatmaps that localize class-specific features against their competitors.

\subsection{Theoretical Advantages}
DiffGradCAM is (a) \textbf{boundary-aligned}, operating on logit gaps that match the softmax decision surface, (b) \textbf{robust}, because inflating a single logit does not alter those gaps, (c) \textbf{discriminative}, highlighting pixels that separate the target class from its rivals. It is also \textbf{simplex-consistent}, living in the $(C-1)$-dimensional log-odds space of the probability simplex, and (d) \textbf{plug-and-play}, since the choice of backpropagation target does not change the network or data.

\subsection{Application to Higher-order Derivative CAMs}
GradCAM++ generalizes GradCAM by replacing a single global weight per channel with pixel-wise coefficients that depend on higher-order derivatives, improving localization when multiple instances of a class are present. Concretely, for a target class 
$c$, GradCAM++ forms per-location coefficients $\alpha_k^{(u,v)}$ from second and third order partials of the target with respect to activations $A_k(u,v)$ and then aggregates. We introduce DiffGradCAM++ which follows this procedure with target difference logit $\Delta^c$:

\begin{equation}
w^{(c)}_k = \sum_{u,v}\alpha_k^{(u,v)}\mathrm{ReLU}\left(\frac{\partial \Delta^c}{\partial A_k(u,v)}\right)
\end{equation}

\begin{equation}
\mathrm{DiffGradCAM}\text{++}^c = \mathrm{ReLU}\left(\sum_k w^{(c)}_k A_k\right)
\end{equation}

As with DiffGradCAM, we name a specific CAM with the prefix of the aggregator used (\eg, MeanDiffGradCAM++).

\section{Experiment Design}
\label{sec:exp_design}

We conduct three experiments: (a) train salience-based models for the few-class Iris PAD task with and without SHAM (addressing {\bf RQ1}); (b) generate and quantitatively compare GradCAM with several DiffGradCAM variants, and GradCAM++ with DiffGradCAM++ variants, using ImageNet-pretrained models (addressing {\bf RQ2}); and (c) compare various CAMs from the clean models in (b) with those from SHAM-fine-tuned counterparts (addressing {\bf RQ3}).

To demonstrate the broad applicability of adversarial SHAM and DiffGradCAM we evaluate on two image classification problems: the seven-class Iris PAD domain and ImageNet (1000 classes).

\subsection{Training Scenarios and Performance Metrics}
\label{sec:training-scenarios}


\textbf{Iris PAD (small $C$).} Following Boyd et al.~\cite{boyd2022human}, we train five DenseNet runs under two supervision regimes: human salience and adversarial SHAM—using an objective function combining classification (cross-entropy) and salience (MSE of target and model salience). We report accuracy on balanced classes and visualize representative CAMs.


\textbf{ImageNet similarity.} Using four ImageNet-pretrained architectures (DenseNet-121, ResNet-50, Inception-v3, and ConvNeXt-Tiny), we sample five validation images per class (5{,}000 total). We generate GradCAM and GradCAM++, along with DiffGradCAM (mean, max, LSE) and the corresponding DiffGradCAM++ variants. For each image and architecture, we compute the per-pixel MSE between (i) DiffGradCAM maps and their baseline GradCAM maps, and (ii) DiffGradCAM++ maps and their baseline GradCAM++ maps. All heatmaps are min–max normalized to $[0,1]$ prior to comparison. We report means across the dataset and use the Wilcoxon rank–sum test  to confirm significance.


\textbf{Susceptibility test.} Using the models and test set from (b), we evaluate robustness across all four architectures. For each model, we compute CAMs for GradCAM \cite{gradcam}, GradCAM++ \cite{gradcam++}, EigenCAM \cite{eigencam}, HiResCAM \cite{hirescam}, XGradCAM \cite{xgradcam}, ScoreCAM \cite{scorecam}, DiffGradCAM (mean, max, LSE), DiffGradCAM++ (mean, max, LSE). We then fine-tune each ImageNet-pretrained backbone for one epoch with SHAM adversarial salience (as in~\cite{heo2019fooling}) and recompute CAMs. For each sample, CAM type, and architecture, we measure the MSE between maps from the clean and SHAM-tuned models (lower is better), average over the dataset, and compare means using the Wilcoxon rank–sum test.

\subsection{Inapplicability of Insertion--Deletion Under Passive Fooling}
\label{subsec:id_not_diagnostic}

A common faithfulness evaluation perturbs images by removing (deletion) or adding
(insertion) regions ranked by a saliency map and measures the area under the model
confidence curve \cite{risepetsiuk2018}. However, our setting is passive fooling: model parameters are optimized so that predictions are preserved on natural inputs, while the gradients and attributions are altered. Our SHAM benchmark instantiates this threat by driving the explanation to a target pattern without degrading accuracy or entropy.

Insertion--deletion judges an explanation's sensitivity to counterfactual pixel perturbations. Passive fooling judges an explanation's vulnerability to manipulated activations. This renders insertion--deletion as an unfit metric for this paper.

\subsection{Experiment Parameters and Compute Resources}
\label{ssec:exp_params}


\textbf{Setup.} For Iris PAD, models are trained for 50 epochs with SGD (lr$=0.002$), equal CE/MSE weights, and five random seeds at one hour per seed. For the SHAM test, we fine-tune each ImageNet-pretrained backbone (DenseNet, ResNet, Inception, and ConvNeXt) for one epoch with SGD requiring five hours each. All experiments run on a single NVIDIA RTX~A6000; compute scales approximately linearly with the number of architectures evaluated (four in our setup). The total GPU hours are thus $1\times5+4\times5=25$. See Table \ref{tab:runtimes} to see how long generating each CAM variant takes.

\begin{table}
  \centering
  \begin{tabular}{
    l
    S[table-format=4.3]  
    @{\;\;\(\pm\)\,}       
    S[table-format=2.3]  
  }
    \toprule
    \textbf{CAM variant} &
    \multicolumn{2}{c}{\textbf{Run time} (\si{\milli\second})} \\
    \midrule
    GradCAM             &   51.981  &  0.959   \\
    EigenCAM            &   78.558  &  1.986   \\
    HiResCAM            &   57.269  &  0.558   \\
    XGradCAM            &   54.339  &  0.298   \\
    ScoreCAM            & 1160.071  &  54.428  \\
    MeanDiffGradCAM     &   49.734  &  0.387   \\
    MaxDiffGradCAM      &   49.461  &  0.340   \\
    LSEDiffGradCAM      &   49.410  &  0.376   \\
    \midrule
    GradCAM++           &   49.057  &  0.354   \\
    MeanDiffGradCAM++   &   49.374  &  0.435   \\
    MaxDiffGradCAM++    &   49.438  &  0.274   \\
    LSEDiffGradCAM++    &   53.507  &  40.126  \\
    \bottomrule
  \end{tabular}
  \caption{Average run times to generate a single CAM of each variant considered in this study ($n=110$, with the first $10$ discarded).}
  \label{tab:runtimes}
\end{table}

\subsection{Datasets}
\label{ssec:datasets}

When referring to ImageNet, we use ImageNet2012 \cite{imagenet2012}. For Iris PAD we use the training, validation, and testing partitions published in \cite{boyd2022human} with resampling to make all attack types equal. 

The Iris PAD training set consist of 193 images from each of these seven classes for a total of $1,351$ samples: Real Iris \cite{real_iris1,real_artificial_textured_print_24,diseased38,real22,real_textured20,real_textured46,real43,real_textured45,boyd2022human}, Artificial \cite{real_artificial_textured_print_24,boyd2022human}, Textured Contacts \cite{real_artificial_textured_print_24,real_textured20,real_textured46,real_textured45,boyd2022human}, Post-Mortem \cite{post40}, Printouts \cite{real_print12,real_artificial_textured_print_24,print21}, Synthetic \cite{synth44}, and Diseased \cite{diseased38}. The validation consist of 500 set-disjoint images from each of the same seven classes for a total of $3,500$ samples. The test set consist of $11,592$ set-disjoint images from the seven classes for a total of  $81,144$ samples. The beneficial human salience for the salience-based training was provided by the authors of \cite{boyd2022human}.

\section{Results}
\label{sec:results}

\begin{figure*}
    \centering
    \begin{subfigure}{0.4\textwidth}
        \centering
        \includegraphics[width=\linewidth]{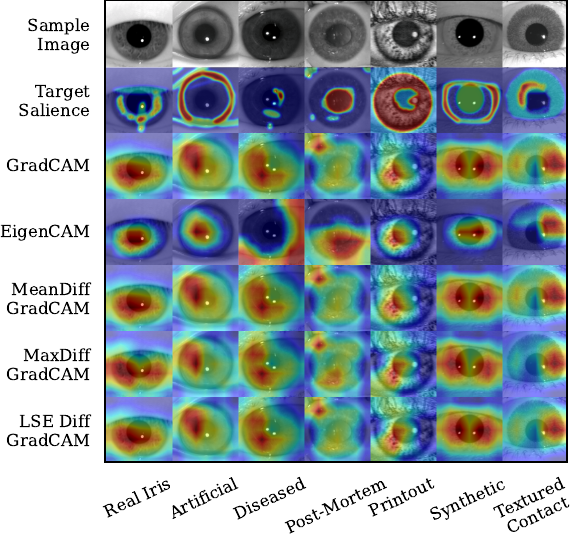}
        \caption{Salience produced by a model trained without SHAM on seven-class Iris PAD. As Iris PAD is not a trivial task, "Target Salience" shows an aggregate annotation of important features as determined by three human Iris PAD experts.} 
        \label{fig:sub_iris_human}
    \end{subfigure}
    \hfill
    \begin{subfigure}{0.4\textwidth}
        \centering
        \includegraphics[width=\linewidth]{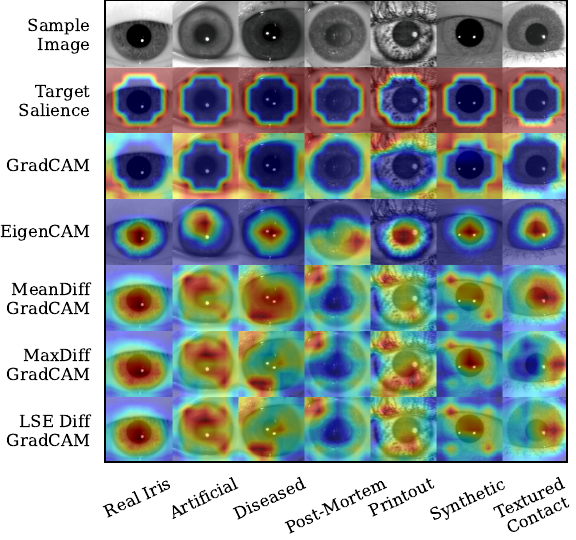}
        \caption{Salience produced by a model trained with SHAM on seven-class Iris PAD. "Target Salience" shows the adversarial SHAM used in training overlaid on the image. CAM types susceptible to SHAM training will more closely resemble "Target Salience."}
        \label{fig:sub_iris_sham}
    \end{subfigure}

    \vspace{1em}

    \begin{subfigure}{0.4\textwidth}
        \centering
        \includegraphics[width=\linewidth]{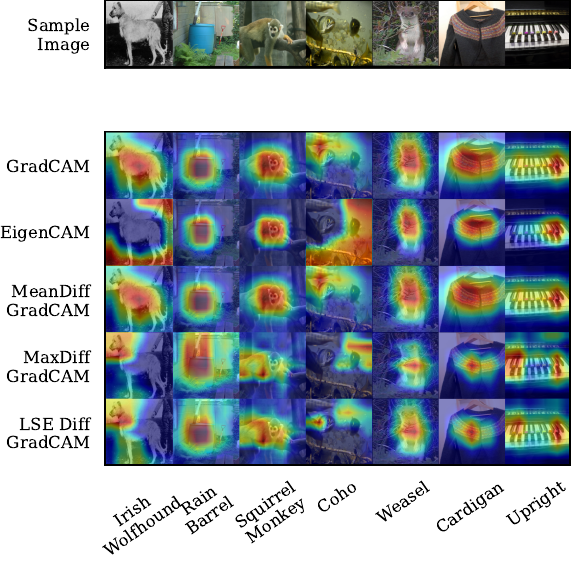}
        \caption{Salience produced by a model trained without SHAM on 1,000-class ImageNet. Classifying ImageNet is human interpretable, thus "Target Salience" is omitted. In this non-adversarial case, a better DiffGradCAM more closely resembles GradCAM.}
        \label{fig:sub_imagenet_normal}
    \end{subfigure}
    \hfill
    \begin{subfigure}{0.4\textwidth}
        \centering
        \includegraphics[width=\linewidth]{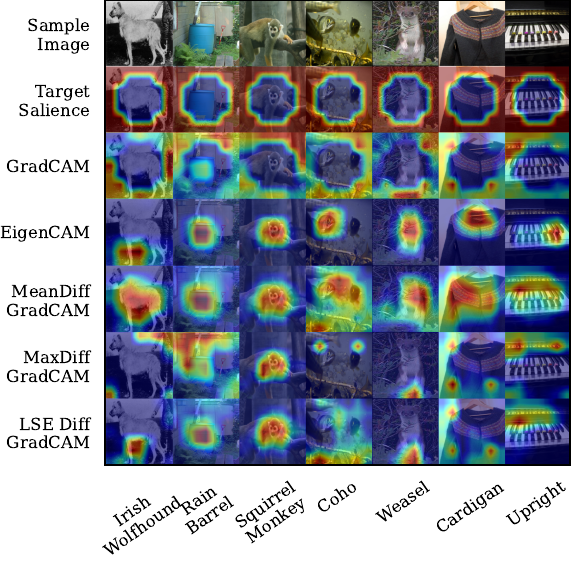}
        \caption{Salience produced by a model trained with SHAM on 1,000-class ImageNet. "Target Salience" shows the adversarial SHAM used in training overlaid on the image. CAM types susceptible to the adversarial SHAM training will more closely resemble "Target Salience."}
        \label{fig:sub_imagenet_sham}
    \end{subfigure}

    \caption{Qualitative analysis of the CAM types examined in this paper on datasets with small and large number of classes and with and without adversarial SHAM interference. We showcase base GradCAM, class-agnostic EigenCAM, and our aggregate class DiffGradCAM on DenseNet architecture. SHAM serves as the passive fooling benchmark and CAMs altered by SHAM inclusion are considered fooled.}
    \label{fig:sham_comparison}
\end{figure*}

\subsection{\textbf{Answering RQ1 (In the novel, few-class Iris PAD domain, can SHAM-based passive fooling produce misleading CAMs without negatively impacting classification accuracy?)}}

\vskip-3mm
\begin{table}[!h]
    \centering
    \footnotesize
    \caption{AUROC performance on Iris PAD dataset. Mean accuracy scores and standard deviations for the balanced classes are shown across independent train-test runs as specified in Sec. \ref{sec:training-scenarios}.}
    \begin{tabular}{l c}
        \toprule
        \textbf{Salience-based Model} & \textbf{Accuracy Score ($\uparrow$)} \\
        \midrule
        Trained with beneficial human salience         & $0.9723 \pm 0.0021$ \\
        Trained with adversarial SHAM salience        & $0.9766 \pm 0.0018$ \\
        \bottomrule
    \end{tabular}
    
    \label{tab:irispad_auroc}
\end{table}

Quantitatively, in Table \ref{tab:irispad_auroc}, we see that the models trained with adversarial SHAM salience do not perform worse than those trained with beneficial (human) salience. Qualitatively, we see in Fig. \ref{fig:sub_iris_human} that in the non-adversarial context there is no significant difference between GradCAM and the DiffGradCAM variants. In Fig. \ref{fig:sub_iris_sham}, we see that GradCAM has been altered to match the adversarial SHAM, but in this few-class classification task all three considered DiffGradCAM highlight similar and relevant features. Previous state-of-the-art EigenCAM, while focusing on different features than DiffGradCAM, also does not match the SHAM salience.

\textbf{Hence, the answer to RQ1 is affirmative:} In the novel, few-class Iris PAD domain, SHAM-based passive fooling produces misleading CAMs without negatively impacting classification accuracy.

\begin{table*}[ht]
    \centering
    \caption{Similarity scores (specified in \ref{sec:training-scenarios}). For each architecture, we report MSE between DiffGradCAM and its baseline GradCAM, and between DiffGradCAM++ and its baseline GradCAM++ (lower is better).}
    \begin{tabular}{l c c c c}
        \toprule
        \textbf{DiffGradCAM} & \textbf{DenseNet}    & \textbf{ResNet}     & \textbf{Inception}  & \textbf{ConvNeXt}  \\
        \midrule
        MeanDiffGradCAM      & $<0.001 \pm <0.001$  & $<0.001 \pm <0.001$ & $<0.001 \pm <0.001$ & $<0.001 \pm <0.001$ \\
        MaxDiffGradCAM       & $0.0852 \pm 0.0761$  & $0.0858 \pm 0.0768$ & $0.0937 \pm 0.1017$ & $0.0307 \pm 0.0424$ \\
        LSEDiffGradCAM      & $0.0615 \pm 0.0719$  & $0.0656 \pm 0.0739$ & $0.0848 \pm 0.1014$ & $0.0059 \pm 0.0217$ \\
        \midrule
        \textbf{DiffGradCAM++} & \textbf{DenseNet}    & \textbf{ResNet}     & \textbf{Inception}  & \textbf{ConvNeXt}  \\
        \midrule
        MeanDiffGradCAM++      & $<0.001 \pm <0.001$  & $<0.001 \pm <0.001$ & $<0.001 \pm <0.001$ & $<0.001 \pm <0.001$ \\
        MaxDiffGradCAM++       & $0.0055 \pm 0.0066$  & $0.0066 \pm 0.0081$ & $0.0029 \pm 0.0046$ & $0.0356 \pm 0.0315$ \\
        LSEDiffGradCAM++      & $0.0038 \pm 0.0050$  & $0.0049 \pm 0.0068$ & $0.0024 \pm 0.0040$ & $0.0084 \pm 0.0176$ \\
        \bottomrule
    \end{tabular}
    
    \label{tab:dgc_vs_gradcam}
\end{table*}

\subsection{\textbf{Answering RQ2 (In the standard, large-scale ImageNet setting, do any of the proposed DiffGradCAM variants match GradCAM in a non-adversarial context?)}}

Qualitatively we see in Fig. \ref{fig:sub_imagenet_normal}, when the number of classes is large, i.e., the contributions from false logits to the difference logit have grown, the DiffGradCAM variants no longer resemble each other and some do not resemble GradCAM. 
However, quantitatively we see in Table \ref{tab:dgc_vs_gradcam} that across all architectures (DenseNet, ResNet, Inception, and ConvNeXt) MeanDiffGradCAM closely resembles GradCAM, indicating that MeanDiffGradCAM is a reliable drop-in replacement for GradCAM on non-adversarial models ($\text{mean MSE} < 10^{-3}$). 
Similarly, the higher-order derivative version, MeanDiffGradCAM++ resembles GradCAM++ with an MSE difference less than $0.001$ on all architectures.
Furthermore, statistical testing indicates that the MeanDiffGradCAM similarity score differs significantly from both MaxDiffGradCAM and LSEDiffGradCAM similarity scores ($p < 0.0001$ with $\alpha=0.05$) and this holds true for MeanDiffGradCAM++ compared to Max- and LSEDiffGradCAM++.

\textbf{Hence, the answer to RQ2 is affirmative:} MeanDiffGradCAM (and MeanDiffGradCAM++) match GradCAM (and GradCAM++) on non-adversarial models (MSE < 1e-3 across 4 backbones). GradCAM mapping is not lost or altered if DiffGradCAMs are used when there is no adversarial attack.

\begin{table*}[ht]
    \centering
    \caption{Susceptibility score (specified in \ref{sec:training-scenarios}) to SHAM-based adversarial fine-tuning (lower is better) across architectures, covering GradCAM, \textbf{GradCAM++}, EigenCAM, HiResCAM, XGradCAM, ScoreCAM, DiffGradCAM (mean, max, LSE), and \textbf{DiffGradCAM++} (mean, max, LSE). }
    \begin{tabular}{l c c c c}
        \toprule
        \textbf{CAM Type} & \textbf{DenseNet}   & \textbf{ResNet}     & \textbf{Inception}  & \textbf{ConvNeXt}  \\
        \midrule
        GradCAM           & $0.1019 \pm 0.0728$ & $0.2318 \pm 0.0699$ & $0.2549 \pm 0.0584$ & $0.1109 \pm 0.0921$ \\
        EigenCAM          & $0.0724 \pm 0.0531$ & $0.0675 \pm 0.0669$ & $0.0506 \pm 0.0748$ & $0.1734 \pm 0.0747$ \\
        HiResCAM          & $0.1339 \pm 0.0657$ & $0.2165 \pm 0.0688$ & $0.2295 \pm 0.0556$ & $0.0829 \pm 0.0767$ \\
        XGradCAM          & $0.1157 \pm 0.0528$ & $0.2165 \pm 0.0688$ & $0.2295 \pm 0.0556$ & $0.0829 \pm 0.0767$ \\
        ScoreCAM          & $0.1661 \pm 0.0646$ & $0.1556 \pm 0.0762$ & $0.0657 \pm 0.0529$ & $0.1098 \pm 0.0748$ \\
        MeanDiffGradCAM   & $\mathbf{0.0460 \pm 0.0398}$ & $\mathbf{0.0567 \pm 0.0278}$ & $\mathbf{0.0402 \pm 0.0359}$ & $0.1153 \pm 0.0955$ \\
        MaxDiffGradCAM    & $0.1200 \pm 0.0621$ & $0.0914 \pm 0.0452$ & $0.1237 \pm 0.0359$ & $0.0860 \pm 0.0551$ \\
        LSEDiffGradCAM    & $0.1050 \pm 0.0632$ & $0.0804 \pm 0.0412$ & $0.1133 \pm 0.0798$ & $\mathbf{0.0810 \pm 0.0563}$ \\
        \midrule
        GradCAM++         & $0.0299 \pm 0.0220$ & $0.0629 \pm 0.0295$ & $\mathbf{0.0351 \pm 0.0268}$ & $0.1477 \pm 0.1062$ \\
        MeanDiffGradCAM++ & $\mathbf{0.0277 \pm 0.0199}$ & $\mathbf{0.0519 \pm 0.0266}$ & $0.0361 \pm 0.0268$ & $0.1442 \pm 0.1053$ \\
        MaxDiffGradCAM++  & $0.0335 \pm 0.0220$ & $0.0686 \pm 0.0347$ & $0.0456 \pm 0.0314$ & $0.0908 \pm 0.0559$ \\
        LSEDiffGradCAM++  & $0.0309 \pm 0.0211$ & $0.0636 \pm 0.0332$ & $0.0436 \pm 0.0301$ & $\mathbf{0.0906 \pm 0.0596}$ \\
        \bottomrule 
    \end{tabular}
    
    \label{tab:sham_susceptibility}
\end{table*}

\subsection{\textbf{Answering RQ3 (In the large-scale ImageNet scenario, are DiffGradCAM variants resistant to SHAM-based adversarial manipulation, and how does their resistance compare to established class-independent CAM methods?)}}

In Table \ref{tab:sham_susceptibility} we see the quantitative susceptibility of different CAM types to adversarial SHAM training across four architectures. Values represent the mean MSE ± standard deviation between CAMs generated on identical images by models trained with and without SHAM on ImageNet on 5000 samples. A lower MSE indicates higher resistance to manipulation. 

We observe the following. First, MeanDiffGradCAM is most consistently the most resistant method of all the first-order derivative CAMs and MeanDiffGradCAM++ the most resistant for the higher-order derivative CAMs. MeanDiffGradCAM++ attains the lowest error among all methods on DenseNet and ResNet and on ConvNeXt, LSE/MaxDiffGradCAM++ yield the smallest errors. On Inception, GradCAM++ remains the best, although it is essentially tied with MeanDiffGradCAM++ (difference of $\approx 10^{-3}$). With the exception of GradCAM++ with EigenCAM on Inception and Max- with LSEDiffGradCAM++ on ConvNeXt, differences between the best and other methods are statistically significant by Wilcoxon rank–sum ($p < 10^{-3}$, $\alpha = 0.05$), though several are close in magnitude.

Second, the DiffGradCAM++ variants generally reduce susceptibility relative to their DiffGradCAM counterparts: \eg, on DenseNet, ResNet, and Inception. 

Third, the class-agnostic method (\eg, EigenCAM) can be competitive on some backbones but is never the highest performing method and is inconsistent across architectures.  Overall, \textbf{DiffGradCAM++} variants provide the strongest and most consistent resistance to SHAM, with mean-based baselines being the safest default.

\textbf{Hence, the answer to RQ3 is affirmative:} MeanDiffGradCAM(++) is resistant to adversarial SHAM fine-tuning, often matching or exceeding state-of-the-art across backbones.

\section{Limitations}
\label{sec:limitations}

\textbf{Aggregator choice.} Across ImageNet and Iris PAD the mean aggregator generally gives the most stable DiffGradCAM and DiffGradCAM++ maps, matching the variance analysis in Sec.~\ref{sec:dist_choice}. However, on ConvNeXt LSE performed best among DiffGradCAM++ variants.


\textbf{Human-based metrics.} We report MSE-based similarity and susceptibility with rank tests for significance as is fitting for a primary quantitative study. However, as human perception is a key component in determining the usefulness of model explanations, future human studies would further validate utility.

\section{Conclusion}
We introduced DiffGradCAM and DiffGradCAM++, which target logit differences rather than a single logit. This decision-aligned formulation aggregates more of the model’s decision process and enhances robustness to adversarial manipulation. Across Iris PAD and ImageNet, and over four archetures, 
the mean-aggregated variants (MeanDiffGradCAM, MeanDiffGradCAM++) typically are the best match to their respective baselines on clean models while exhibiting lower susceptibility to SHAM. The modifications are plug-and-play and preserve the standard CAM pipeline.

Key takeaways include: (i) mean-based contrastive targets as a reliable drop-in replacement for GradCAM/GradCAM++; (ii) higher-order derivative weighting (the ``++'' family) further stabilizes explanations in multi-instance scenes; and (iii) decision alignment improves robustness beyond mere similarity to a baseline map.

As practical guidance, one may use MeanDiffGradCAM by default for single-pass deployment when GradCAM is the default and prefer MeanDiffGradCAM++ when multiple object instances or tighter localization is required.

The evidence across RQ1–RQ3 indicates that decision-aligned CAMs, especially the mean-aggregated DiffGradCAM++, offer the most consistent robustness to passive saliency manipulation while preserving the desirable behavior of their GradCAM/GradCAM++ baselines.

Future work should explore the application of contrastive principles to other explanation methods and assess DiffGradCAM's utility in high-stakes domains where explanation robustness is particularly critical.

\section{Acknowledgement}
This work was supported by the U.S. Department of Defense (Contract No. W52P1J-20-9-3009). Any opinions, findings, and conclusions or recommendations expressed in this material are those of the authors and do not necessarily reflect the views of the U.S. Department of Defense or the U.S. Government. The U.S. Government is authorized to reproduce and distribute reprints for Government purposes, notwithstanding any copyright notation here on.

{
    \small
    \bibliographystyle{ieeenat_fullname}
    \bibliography{main}
}


\end{document}